\documentclass{article}

\usepackage{microtype}
\usepackage{graphicx}
\usepackage{caption}
\usepackage{subcaption}
\usepackage{booktabs} 
\usepackage{verbatim}

\usepackage{hyperref}

\usepackage[accepted]{icml2026}

\usepackage{amsmath}
\usepackage{amssymb}
\usepackage{mathtools}
\usepackage{amsthm}

\usepackage[capitalize,noabbrev]{cleveref}

\theoremstyle{plain}
\newtheorem{theorem}{Theorem}[section]

\newtheorem{corollary}[theorem]{Corollary}
\theoremstyle{definition}

\theoremstyle{remark}

\usepackage[textsize=tiny]{todonotes}

\icmltitlerunning{From Exact Hits to Close Enough: Semantic Caching for LLM Embeddings}

\begin{document}

\twocolumn[
\icmltitle{From Exact Hits to Close Enough: Semantic Caching for LLM Embeddings}



\icmlsetsymbol{equal}{*}

\begin{icmlauthorlist}
\icmlauthor{Dvir David Biton}{equal,yyy}
\icmlauthor{Roy Friedman}{equal,yyy}
\end{icmlauthorlist}

\icmlaffiliation{yyy}{Department of Computer Science, Technion – Israel Institute of Technology
, Haifa, Israel}

\icmlcorrespondingauthor{Dvir David Biton}{dbiton@campus.technion.ac.il}
\icmlcorrespondingauthor{Roy Friedman}{roy@technion.ac.il}

\icmlkeywords{Machine Learning}

\vskip 0.3in
]



\printAffiliationsAndNotice{}  

\begin{abstract}
The rapid adoption of large language models (LLMs) has created demand for faster responses and lower costs. Semantic caching, reusing semantically similar requests via their embeddings, addresses this need but breaks classic cache assumptions and raises new challenges. In this paper, we explore offline policies for semantic caching, proving that implementing an optimal offline policy is NP-hard, and propose several polynomial-time heuristics. We also present online semantic aware cache policies that combine recency, frequency, and locality. Evaluations on diverse datasets show that while frequency based policies are strong baselines, our novel variant improves semantic accuracy. Our findings reveal effective strategies for current systems and highlight substantial headroom for future innovation. All code is open source.
\end{abstract}

\section{Introduction} 
\label{sec:introduction}

Modern large language models (LLMs) rely on embedding methods to encode the semantic meaning of user queries into high-dimensional vectors. 
These embeddings enable efficient comparison of queries, making them a cornerstone of natural language processing tasks, such as question answering and semantic search.
LLMs have gained wide adoption in the last few years although they suffer from high computational cost, as well as significant consumption of energy, memory, storage, and communication bandwidth.

Caching is a known technique to alleviate such limitations, and has been applied successfully in a variety of domains.
Recently there have been several attempts to apply caching to LLMs~\cite{GPTCache,MeanCache,SSSDCD23}.
By encoding each query into an embedding and comparing it against cached embeddings, semantically similar requests are detected, which enables the system to return the previously computed answer, avoiding redundant inference and reducing overall latency and compute cost.
Most of these works adopt a na\"ive or ad-hoc approach to managing the cache, i.e., for deciding which embeddings should be inserted into the cache, and once the cache is full, which should be evicted to make room for new ones.
This is despite evidence from other domains that clever cache management can significantly improve caching benefits~\cite{ARC,FRD,TinyLFU}.

In particular, it is known that different workloads may yield significantly different hit rates depending on the cache management policy chosen.
This has led to a plethora of cache management policies~\cite{LRU,WLFU,LFUAGING} as well as to various adaptive schemes and learning based schemes.
However, a universally optimal cache management policy remains an open research challenge.

Since semantically similar queries can be answered by the same response, despite the fact that syntactically they are not exactly the same, most LLM caching solutions rely on \emph{semantic caching}, in which a query is considered a hit if the cache already includes the reply to some query such that the embedding vectors of both queries are ``close enough'' under some metric such as L2 norm or cosine similarity.

The impact of this shift from exact match caching to semantic caching on the vast knowledge about caching and cache replacement policies is not clear a priori.
Hence, in this work we take a systematic approach to studying LLM cache management and, in particular, what changes, if any, are needed when moving from exact matches to semantic~caching.

\subsection{Related Work}

Recent works, such as GPTCache~\cite{GPTCache}, have demonstrated the benefits of applying caching
to LLM inference.
GPTCache can be integrated with multiple embedding models, vector stores and cache management schemes.
It supports LRU, LFU, FIFO, and~Random by default.

MeanCache is another recent cache for LLMs~\cite{MeanCache}, which applies lighter models than GPTCache.
The work utilizes LRU as cache replacement, at least by default.
For most of the workloads we have experimented with in this work, LRU is not a good cache management policy.

GPTSemCache also applies caching to avoid repeated LLM API invocations~\cite{GPTSemCache}.
They report a reduction in API calls by up to 68.8\% across
various query categories, with cache hit rates ranging from 61.6\% to 68.8\%, while achieving high accuracy.
GPTSemCache relies on Redis to store cached data without any hard cache size limit, and therefore does not consider any replacement strategy.

vCache~\cite{vCache} employs an online learning algorithm to estimate an optimal threshold for each cached prompt similarity bound.
They assume an infinite cache size, rendering the replacement policy irrelevant.
When the cache size matters, our work becomes complementary to vCache and can be integrated with it.
Similarly, InstCache~\cite{InstCache} and GenerativeCache~\cite{IKKM25} are additional LLM caching mechanisms that assume an infinite cache and therefore do not apply any replacement policy.

The work of~\cite{ZSZBJJ23} studied how to minimize the cost of a cache based LLM system by learning the expected cost of various queries processing time, and integrating it with the Least Expected Cost policy and the Greedy Dual-Size with Frequency (GDSF) replacement algorithm.
They assume an oracle that clusters semantically similar queries into equivalence groups, thereby ignoring the issue of having multiple cached vectors that can be considered a hit for the same query.
In that sense, their work can be compared to our ClusterLFU and CRVB policies, which we have shown to be often suboptimal.

RAGCache~\cite{RAGCache} caches RAG contents.
It applies a novel prefix-aware Greedy-Dual-Size-Frequency replacement policy, and compares it to GDSF, LRU and~LFU.

Proximity~\cite{BJKPPRV25} is another RAG cache project.
The focus in Proximity is reducing the cache search time, which is obtained by dividing the cache into independent buckets, or sets, and applying a mapping mechanism like locality-sensitive hashing (LSH) that ensures that semantically similar vectors are placed in the same bucket.
Then each bucket can be managed using its own policy.
Currently, Proximity supports LRU and FIFO.
Our work is therefore complementing to~\cite{BJKPPRV25}, and our results can be applied for better selection of the replacement policies of the buckets of Proximity.

Let us note that caching high-dimensional vectors for similarity search is not exclusive to NLP. It is also used in many applications such as recommender systems, reverse image search, and biological sequence analysis.

\subsection{Summary of Contributions}
We show that the optimal clairvoyant caching algorithm, also known as Belady’s OPT or MIN, is no longer optimal in the context of semantic caching, and that computing the corresponding clairvoyant solution, which we term VOPT, is NP-hard.
To address this, we propose and analyze three polynomial-time clairvoyant heuristics: Clustered Relaxed Vector Belady (CRVB), Volume Greedy Relaxed Vector Belady (FGRVB), and Time Greedy Relaxed Vector Belady (RGRVB).
We compare these clairvoyant approaches against a broad range of established online cache management policies, adapted for semantic caching, across nine publicly available workloads.
Our results show that frequency-based algorithms generally outperform other online approaches, with our novel SphereLFU achieving the highest accuracy.
Also, in most workloads, the clairvoyant algorithms significantly outperform online policies, suggesting substantial remaining opportunities for~innovation.


\section{Preliminaries}
\label{sec:sem-cache}

\subsection{Semantic Caching}


%

A \emph{semantic (vector) cache} $C$ is parameterized by ($i$) $Dim$ - the dimensionality of each vector, ($ii$) $N$ - the maximum number of vectors the cache can store, and ($iii$) $D_{\text{thresh}}$ - a distance threshold. A stored vector within this distance from a request is considered a \emph{cache hit}.
Each stored vector is a tuple of $Dim$ floats. In our evaluation, all vectors are L2-normalized to have unit magnitude, although this is not strictly required.
%
%
The cache exposes two basic methods:
\begin{description}
    \item[Query:] Given a vector $V$ and an integer $M$, return the $M$ or less vectors that are closest to $V$, restricted to those within a distance smaller than $D_{\text{thresh}}$.
    \item[Update:] Given a vector V, the cache is requested to store V. If the cache already contains $N$ vectors, inserting the new vector requires evicting one of the stored vectors from the cache. The cache's management policy decides which vectors are stored and evicted.
\end{description}

For efficiency reasons, our cache implementation supports batches, i.e., querying for multiple vectors using a single request and caching multiple vectors with a single~request.

\subsection{Embedding}
\label{sec:embedding}

Embedding methods aim to project high-dimensional textual data into vector representations that preserve semantic relationships. Sentence-BERT (SBERT)~\cite{SBERT} is a widely used model of this kind, which offers a balanced tradeoff between semantic clustering quality and runtime efficiency. In our experiments, we normalize the embeddings so that their magnitude equals 1, and specifically use the all-MiniLM-L6-v2 variant, which produces 384-dimensional vectors.

\subsection{Embeddings Similarity}
\label{sec:embeddings similarity}
Embedding similarity is commonly measured using cosine similarity, with L2 distance serving as a closely related alternative~\cite{FAISS}. 

In this work, we compute similarity using  L2 distance. Our embeddings are normalized to have a magnitude of 1, therefore cosine similarity and L2 distance become equivalent up to a monotonic transformation, and the choice between them has no effect on the nearest neighbor results.

For normalized embeddings, the L2 distance directly reflects angular similarity. Identical embeddings have distance $0$, while completely opposite embeddings have distance $2$. A distance of $\sqrt{2}$ corresponds to orthogonal vectors, i.e., unrelated vectors. Distances between $0$ and $\sqrt{2}$ indicate increasing alignment: for example, an euclidean distance of 0.5 corresponds to a cosine similarity of about 0.875, meaning the vectors are close but not identical. Conversely, distances beyond $\sqrt{2}$ indicate increasing opposition, with values approaching $2$ representing strong dissimilarity.

In this work, we utilized three normalized L2 distance thresholds ($D_{thresh}$): 0.5, 0.7, and 0.9, which correspond to cosine similarity scores of approximately 0.88, 0.75, and 0.60, respectively. While similarity thresholds are data dependent, we selected these values based on established benchmarks: 0.88 targets strict near-duplicate detection, 0.75 serves as a balanced baseline for semantic equivalence, and 0.60 functions as a high-recall filter for broader topical relevance~\cite{SBERT, SemEval}.

\section{Theoretical Analysis}
\label{sec:rvb}

\subsection{OPT}
\label{sec:opt}

OPT, also called Belady’s optimal or MIN, is the clairvoyant cache policy that, upon a miss with a full cache, evicts the resident item whose next request occurs farthest into the future~\cite{BP74}. This rule provably minimizes total misses for any fixed capacity on a given finite request trace. Classic exchange-argument proofs show that any deviation cannot improve the miss count. Because it requires perfect future knowledge, OPT is not implementable online, but it serves as an oracle upper bound for evaluation and as a target for online policies. Computing OPT for a cache sized $N$ and a known in advance trace of length $n$ requires $O(nlogN)$. OPT defines the achievable frontier against which real policies are compared.

As stated before, in semantic caching, a hit occurs if the requested vector lies within $D_{\text{thresh}}$ of a cached vector.
This heuristic breaks OPT’s optimality: consider a workload with two disjoint regions of requests, one dense and popular and the other sparse and unpopular. An OPT policy translated directly would fill the cache with vectors from the dense region, although only a handful are needed to cover that region. On the other hand, the sparse region will suffer unnecessary misses. This means we must redefine an optimal algorithm for semantic caching.

\subsection{VOPT}

We define VOPT as an offline policy that achieves the maximum hit rate for a given semantic cache and workload. Different policies may achieve the maximum hit rate, and we would consider all of them to be an instance of VOPT.

\begin{theorem}[Hardness of Computing VOPT]
\label{thm:vopt_mcp}
Let $D_{\text{thresh}} > 0$ and let $k$ be the cache capacity.  
Given a request sequence of vectors, computing VOPT is NP-hard.
\end{theorem}

\begin{proof}
We reduce from the Maximum Coverage Problem (MCP).  
In MCP, we are given sets $S_1, \dots, S_m$ over a ground set $U = \{e_1,\dots,e_n\}$ and an integer $k$.  
The goal is to choose at most $k$ of the sets whose union covers the largest number of elements~\cite{MCP}.

\paragraph{Requests.}
We construct vectors $u_S$ (sets) and $v_e$ (elements) in $\mathbb{R}^{m+n}$ in polynomial time (see Appendix~\ref{sec:vector_construction}). For distinct sets $S, S'$ and elements $e, e'$:
\begin{enumerate}
    \item $d(u_S, u_{S'}) > D_{\text{thresh}}$ and $d(v_e, v_{e'}) > D_{\text{thresh}}$.
    \item $d(u_S, v_e) \le D_{\text{thresh}} \iff e \in S$.
\end{enumerate}

\paragraph{Request sequence.}
We define a two-phase request sequence:

\begin{itemize}
    \item \textbf{Phase 1:} Request all set vectors $v_{S_1},\dots,v_{S_m}$.  
    All requests are misses, but the policy may keep up to $k$ set vectors after Phase~1.

    \item \textbf{Phase 2:} Request all element vectors $v_{e_1},\dots,v_{e_n}$.  
    A request to $v_{e_i}$ is a hit iff some cached $v_{S_j}$ corresponds to a set containing $e_i$.
\end{itemize}

\paragraph{Correctness.}
The number of hits in Phase~2 equals the number of elements in the union of the $k$ sets associated with the cached set vectors.  
Thus, an optimal VOPT policy chooses the same $k$ sets that achieve optimal maximum coverage.
Since MCP is NP-hard, so is computing VOPT.
\end{proof}

\begin{corollary}[Inapproximability of VOPT]
\label{cor:vopt_inapprox}
Unless $\mathrm{P}=\mathrm{NP}$, no polynomial-time algorithm can approximate VOPT's hit ratio within a factor better than $(1 - 1/e)$.
\end{corollary}

\begin{proof}
The reduction in Theorem~\ref{thm:vopt_mcp} is approximation-preserving: the optimal number of Phase~2 hits equals the maximum coverage value of the MCP instance.  
MCP is known to be NP-hard to approximate to within a factor better than $(1 - 1/e)$~\cite{MCP}.  
Thus, any algorithm approximating VOPT more closely would imply an algorithm approximating MCP more closely, contradicting known hardness results.
\end{proof}

\begin{theorem}[Well-Definedness and Computational Infeasibility of VOPT]
\label{thm:vopt_defined}
For any request sequence and threshold distance $D_{\text{thresh}}$, VOPT is well-defined: an optimal selection of up to $k$ cached vectors that maximizes semantic hits always exists.  
However, computing VOPT requires solving an NP-hard optimization problem and is therefore computationally infeasible for realistic workloads.
\end{theorem}

\begin{proof}
Well-definedness follows from the fact that the request sequence is finite and there are finitely many subsets of size at most $k$ that the cache may contain.  
Each such subset induces a unique number of semantic hits, and an optimal choice exists.
\end{proof}

\begin{corollary}[Approximability of the Static Formulation]
\label{cor:static_approx}
Consider the static variant of the problem where the cache contents $S$ ($|S| \leq k$) are selected and inserted prior to processing the request sequence and remain fixed. In this setting, the objective function is monotone submodular, and a greedy selection strategy achieves an approximation ratio of $(1 - 1/e) \approx 0.632$.
\end{corollary}

The reduction from the Maximum Coverage Problem (MCP) in Theorem~\ref{thm:vopt_mcp} reveals a fundamental structural property: if the cache admits no evictions (i.e., the content is static), the problem of maximizing hits is mathematically equivalent to MCP. The objective function of MCP exhibits diminishing returns, a property formally known as submodularity. It is known~\cite{submod} that for such problems, iteratively selecting the element that covers the most uncovered items yields the optimal $(1 - 1/e)$ approximation.

However, this strict theoretical guarantee does not extend to the general (dynamic) semantic caching problem. In dynamic settings, an eviction decision at time $t$ constrains the feasible cache states at time $t+1$, creating temporal dependencies that break the submodularity of the global objective function. Nevertheless, the static relaxation serves as the theoretical foundation for our offline heuristics. By treating the dynamic problem as a sequence of local maximum coverage problems, where we attempt to maximize the number of future hits covered by the current cache state, we leverage the strong performance of greedy strategies on submodular functions, even in the absence of worst-case guarantees.

\subsection{VOPT heuristics}

We propose 3 heuristics of VOPT.
In our measurements, as reported below, we have found that different ones obtain slightly higher or lower hit rate based on the workload and cache size.
Obviously, since these are offline schemes, meant to establish an upper bound on the best attainable performance, and they are computationally efficient, one can simply implement all three and take the best one at any given data point.

\subsubsection{Clustered Relaxed Vector Belady}
\label{sec:CRVB}

As mentioned before, the challenge in adapting OPT to semantic caches is that multiple cached vectors may cover the same future requests, leading to coverage redundancy.
One method to utilize our clairvoyant knowledge of all requests is to divide them into semantically identical clusters, where the maximum distance between two vectors in the same cluster is $D_{\text{thresh}}$.
To create the clusters, we first construct a graph $G$ in which each node represents one embedding vector.
An edge is added between two nodes if their embeddings are semantically identical, i.e., if their distance is below the threshold $D_{\text{thresh}}$.

Since any cluster of semantically identical embeddings corresponds to a clique in $G$, the task of minimizing the number of clusters reduces to finding the Minimum Clique Cover (MCC). This problem is NP-hard, and can be solved by coloring the complement graph of $G$, denoted~$\overline{G}$~\cite{MCC_Solve}. 
A useful heuristic is greedy coloring with DSATUR, which strikes a balance between accuracy and runtime. However, we found that for this use case, the complement graph is too large and dense for the coloring method to be practical. 
Instead, we adopt a simple greedy approach: iteratively extracting maximal cliques. 
While this method yields lower quality partitions, it produces sufficiently good results in significantly lower runtime, particularly when the graph is sparse as in our case.

After clustering, the problem reduces to classic caching with exact matching, where each embedding is represented by its cluster ID. In this setting, OPT is provably the optimal policy and thus achieves the highest possible hit rate among policies that rely solely on clusters. If embeddings form true equivalence classes, i.e., if the similarity relation is transitive ($d(v_i, v_j) \leq D_{\text{thresh}} \land d(v_i,v_k) \leq D_{\text{thresh}} \;\Rightarrow\; d(v_j, v_k) \leq D_{\text{thresh}}$), then clustering perfectly captures semantic identity, and CRVB yields optimal results. 
A further consequence of this property is that finding the Minimum Clique Cover becomes trivial, since each connected component is a clique.

In practice, we encountered cases where a vector $v_i$ is semantically equivalent to both $v_j$ and $v_k$, but $v_j$ and $v_k$ are not equivalent.
Thus, CRVB is not optimal in the general case, and is used as a baseline oracle.

\subsubsection{Frequency Greedy Relaxed Vector Belady}
\label{sec:FGRVB}

FGRVB attempts to approximate the optimal policy by focusing on the quantity of future hits. This approach is motivated directly by the reduction to MCP.

Although we cannot guarantee the $(1-1/e)$ bound in the dynamic setting, FGRVB implements the greedy heuristic for submodular maximization as an online eviction policy. FGRVB ranks cached embeddings by their marginal contribution to the total number of future hits. On a cache miss, FGRVB counts the number of future unique hits attributed to the missed request, i.e., hits that are only covered by said embeddings. It replaces the cached item with the least unique covered future hits, given that it covers fewer hits compared to the missed item.

Let $R_{t:} = (r_t, r_{t+1}, \dots)$ be the sequence of all future requests. For a cached vector $v$, we define its volume score as the cardinality of the subset of future requests it covers:
\begin{equation}
    \text{Vol}(v) = |\{ r_i \in R_{t:} : d(v, r_i) < D_{\text{thresh}} \}|
\end{equation}

To handle redundancy, a strict implementation calculates the \emph{marginal} gain, i.e., how many future requests $v$ covers that are not already covered by other cached items. Upon a cache miss with a full cache, FGRVB evicts the vector with the lowest marginal volume score. In essence, the policy attempts to maintain the set of $k$ vectors that "cover" the maximum possible number of remaining requests in the trace, regardless of when those requests occur.

\subsubsection{Recency Greedy Relaxed Vector Belady}
\label{sec:GRVB}

To address the fact that vectors from one cluster may also cover vectors in another, we developed RGRVB. The core idea is based on the cover relation, which expresses the potential future hits of an embedding vector:
$$
\text{cover}(r_i) = \{\, r_j \;:\; d(r_i,r_j) < D_{\text{thresh}}, \; j > i \,\},
$$
where $d(\cdot,\cdot)$ denotes the embedding distance. For each embedding vector $r_i$ at time $t$ (the index of the last request made to the cache), we define
$$
\text{NextCover}(r_i,t) = \min \{\, j \;:\; j > t,\; r_j \in \text{cover}(r_i) \,\},
$$
with $\text{NextCover}(r_i,t) = \infty$ if no such $j$ exists. The cache cover is then
$$
\text{CacheCover}(C,t) = \{\, \text{NextCover}(r_j,t) \;\mid\; r_j \in C \,\}.
$$
While the cache is not yet full, every embedding vector is admitted. After the cache reaches capacity, a new embedding $r_i$ is accepted only if its next cover is not already served:
$$
\text{NextCover}(r_i,t) \notin \text{CacheCover}(C,t).
$$
If admitted, $r_i$ replaces the embedding vector whose next hit lies furthest in the future. We consider the next hit rather than all future hits; while embeddings may cover requests far ahead, by the time those requests occur the cache has changed enough to make such predictions unreliable. Also, empirically, focusing on the next hit yields better results.




\begin{figure*}[ht!]
\begin{subfigure}{0.9\linewidth}
    \includegraphics[width=\linewidth]{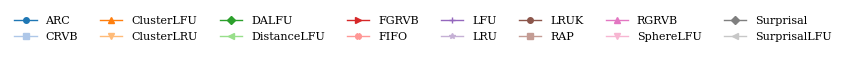}
    \label{fig:legend-0}
\end{subfigure}
  \vspace{0.5em}
  \centering
  \begin{subfigure}{0.31\linewidth}
    \includegraphics[width=\linewidth]{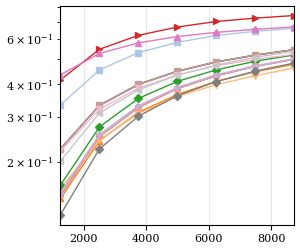}
    \caption{ELI5}
    \label{fig:ELI5}
  \end{subfigure}\hfill
  \begin{subfigure}{0.31\linewidth}
    \includegraphics[width=\linewidth]{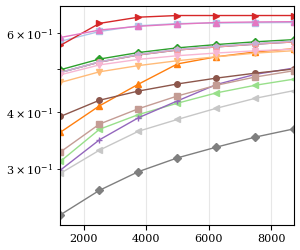}
    \caption{WildChat}
    \label{fig:WildChat}
  \end{subfigure}\hfill
  \begin{subfigure}{0.31\linewidth}
    \includegraphics[width=\linewidth]{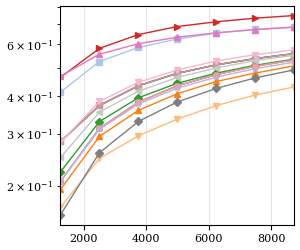}
    \caption{NaturalQuestions}
    \label{fig:NaturalQuestions}
  \end{subfigure}
  \vspace{0.5em}
    \begin{subfigure}{0.31\linewidth}
    \includegraphics[width=\linewidth]{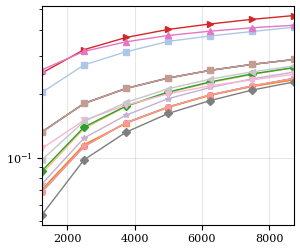}
    \caption{HotPotQA}
    \label{fig:HotPotQA}
  \end{subfigure}\hfill
  \begin{subfigure}{0.31\linewidth}
    \includegraphics[width=\linewidth]{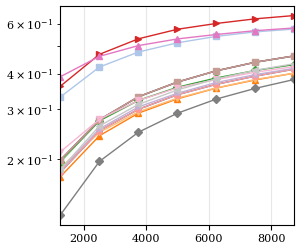}
    \caption{TriviaQA}
    \label{fig:TriviaQA}
  \end{subfigure}\hfill
  \begin{subfigure}{0.31\linewidth}
    \includegraphics[width=\linewidth]{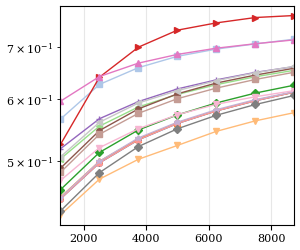}
    \caption{Quora}
    \label{fig:Quora}
  \end{subfigure}
  \vspace{0.5em}
  \begin{subfigure}{0.31\linewidth}
    \includegraphics[width=\linewidth]{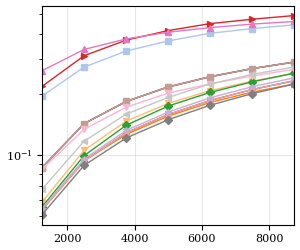}
    \caption{MsMarco}
    \label{fig:MsMarco}
  \end{subfigure}\hfill
  \begin{subfigure}{0.31\linewidth}
    \includegraphics[width=\linewidth]{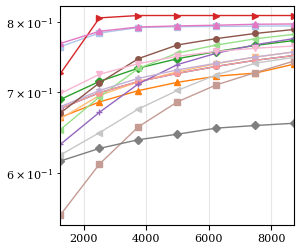}
    \caption{MMLU}
    \label{fig:MMLU}
  \end{subfigure}\hfill
  \begin{subfigure}{0.31\linewidth}
    \includegraphics[width=\linewidth]{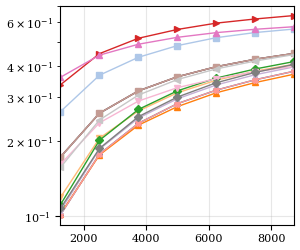}
    \caption{StackOverflow}
    \label{fig:StackOverflow}
  \end{subfigure}
\label{fig:hit-rate}
\caption{Hit rate as a function of cache size - note log scale on the Y-axis}
\end{figure*}

\begin{figure*}[ht!]
  \centering
\begin{subfigure}{0.9\linewidth}
    \includegraphics[width=\linewidth]{figures/legend.png}
    \label{fig:legend-1}
\end{subfigure}
  \vspace{0.5em}
  \begin{subfigure}{0.31\linewidth}
    \includegraphics[width=\linewidth]{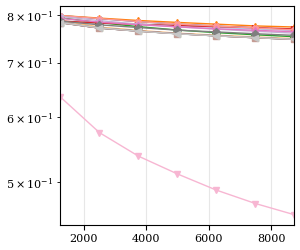}
    \caption{ELI5}
    \label{fig:MHD:ELI5}
  \end{subfigure}\hfill
  \begin{subfigure}{0.31\linewidth}
    \includegraphics[width=\linewidth]{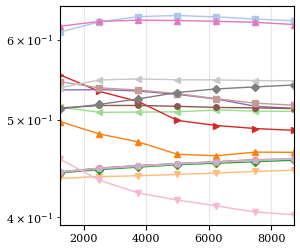}
    \caption{WildChat}
    \label{fig:MHD:WildChat}
  \end{subfigure}\hfill
  \begin{subfigure}{0.31\linewidth}
    \includegraphics[width=\linewidth]{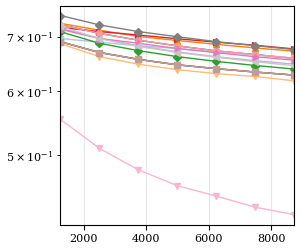}
    \caption{NaturalQuestions}
    \label{fig:fig:MHD:NaturalQuestions}
  \end{subfigure}
  \vspace{0.5em}
    \begin{subfigure}{0.31\linewidth}
    \includegraphics[width=\linewidth]{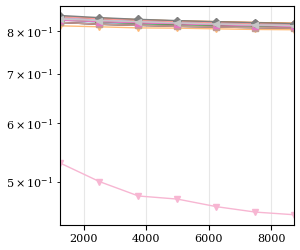}
    \caption{HotPotQA}
    \label{fig:MHD:HotPotQA}
  \end{subfigure}\hfill
  \begin{subfigure}{0.31\linewidth}
    \includegraphics[width=\linewidth]{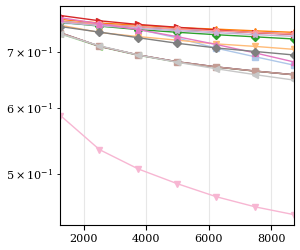}
    \caption{TriviaQA}
    \label{fig:MHD:TriviaQA}
  \end{subfigure}\hfill
  \begin{subfigure}{0.31\linewidth}
    \includegraphics[width=\linewidth]{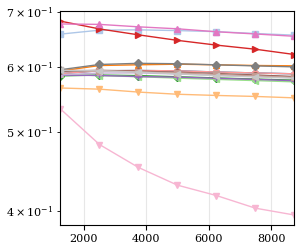}
    \caption{Quora}
    \label{fig:MHD:Quora}
  \end{subfigure}
  \vspace{0.5em}
  \begin{subfigure}{0.31\linewidth}
    \includegraphics[width=\linewidth]{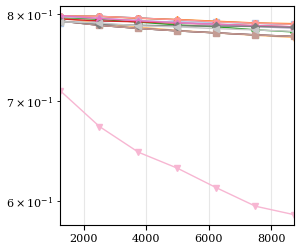}
    \caption{MsMarco}
    \label{fig:MHD:MsMarco}
  \end{subfigure}\hfill
  \begin{subfigure}{0.31\linewidth}
    \includegraphics[width=\linewidth]{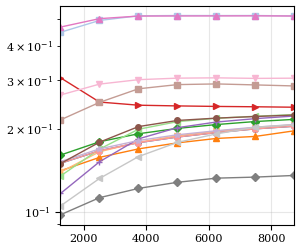}
    \caption{MMLU}
    \label{fig:MHD:MMLU}
  \end{subfigure}\hfill
  \begin{subfigure}{0.31\linewidth}
    \includegraphics[width=\linewidth]{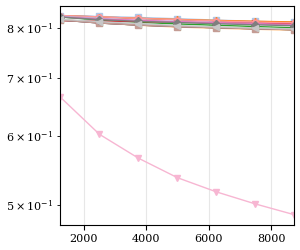}
    \caption{StackOverflow}
    \label{fig:MHD:StackOverflow}
  \end{subfigure}
\label{fig:mean-hit-distance}
\caption{Mean hit distance as a function of cache size - note log scale on the Y-axis}
\end{figure*}

\section{Online Cache Management Policies}

In this work, we have explored immediate adaptations to semantic caching of several known caching policies including LRU~\cite{LRU}, LFU~\cite{WLFU}, LFUDA~\cite{LFUAGING}, LRU-K~\cite{LRUK}, ARC~\cite{ARC}, and RAP~\cite{RAP}.
That is, whenever there is a hit, we update the metadata of the top item that is returned by the cache for the respective query according to the respective policy.

In addition, we experimented with several alternative ideas of how to better address the fact that in semantic caching, there may be multiple vectors that could potentially serve as a hit, i.e., are within the threshold distance from the queried vector.
Since in our experiment LFU based policies seems to generally perform better than other approaches, we focus in this paper on LFU variations.
To that end, we have designed SphereLFU, MissLFU, ClusterLFU, ClusterLRU, DistanceLFU, Surprisal, and SurprisalLFU.
For lack of space, we refer the reader to Appendix~\ref{sec:all-policies} for additional information about these policies.

Of these, SphereLFU is particularly notable for achieving the highest semantic accuracy. 
Unlike discrete frequency counters, SphereLFU employs a probabilistic credit assignment mechanism, where usage 'mass' is distributed among neighboring cached vectors based on a similarity kernel (\Cref{sec:SphereLFU}).
In that sense, it is trying act as an online approximation of~FGRVB.


We have explored ML based policies, based on LRB~\cite{LRB}, but found that results were much worse than other methods we explored. This is due to the fact that the workloads were not long enough to enable the complex learning mechanism of LRB to converge, especially with the more complex nature of non-exact matching.

\section{Experiments}
\label{sec:evaluation}

In this section, we evaluate the performance of our proposed semantic caching policies against standard baselines and offline oracles. We utilize the nine datasets described in \Cref{sec:datasets}, processing the first 100K entries of each (91K for NaturalQuestions). Based on the properties observed in our vector analysis and on Section~\ref{sec:embeddings similarity}, we set the semantic threshold distance $D_{\text{thresh}} = 0.9$ for all reported experiments. We selected the threshold such that the hit rate is within a useful range for all datasets, at least 10\%, although the trends are similar for other thresholds. In Appendix~\ref{sec:additional-experiments}, we present results for the smaller normalized L2 distance thresholds of 0.7 and 0.5. In addition, the complete unprocessed data is available online~\cite{GitHub}.

\subsection{Datasets}
\label{sec:datasets}
The nine real-world datasets we employed in our evaluation include MsMarco~\cite{MSMARCO}, WildChat~\cite{WildChat}, ELI~\cite{ELI5}, Natural Questions (NQ)~\cite{NaturalQuestions}, StackOverflow queries (SO)~\cite{StackOverflowBigQuery}, Quora~\cite{QuoraQuestionPairs}, MMLU~\cite{MMLU}, TriviaQA~\cite{TriviaQA}, and HotPot QA~\cite{HotpotQA}.
They span chats, conversational reformulations, community Q\&A, web search, open-domain QA, and broad consumer questions.

We have obtained the datasets from Hugging Face and processed them into a unified format, utilizing relevant data from each for specific purposes detailed below. 
Our canonical datasets can be generated with a script that is included in our open-source project online.
We only embed queries from the datasets, and not~responses.

An analysis of these datasets' characteristics appears in Appendix~\ref{sec:all-datasets}.
Our evaluation has revealed that they have the following properties: they rarely carry ordering signal, mostly affected by frequency and follow a specific distribution. It is of interest that when plotting the point density for the embeddings and looking at the number of neighbors of distance $< D$, regardless of the selected $D$, the frequency vs rank graph follows a Zipfian distribution. 

\subsection{Implementation and Hardware}

We have implemented our code in Python, utilizing well known, open source modules such as NumPy, SciPy, etc for efficiency. For this paper, we utilized a Faiss flat index for our vector store, although our code also supports other vector stores such as Milvus and HNSWlib.
Our code is available in open-source~\cite{GitHub}.

\subsection{Hit Rate Results}

Most datasets exhibit a strong frequency bias, which can be inferred from the 3 bands formed: for offline policies, frequency based policies and the other temporal based policies. Consequently, LFU-based policies generally outperform recency-based methods such as FIFO and LRU. Among online policies, SphereLFU consistently achieves high hit rates, rivaling other frequency-based baselines. On ``hard'' datasets with long-tail distributions, such as StackOverflow and HotPotQA, SphereLFU remains the top online contender, though a noticeable gap remains between it and the theoretical upper bound obtained by the offline heuristics of~VOPT.

This gap confirms that in low-signal regimes, clairvoyant knowledge provides a decisive advantage over online estimation. However, in non-stationary workloads like WildChat, high temporal locality allows LRU or time-decayed density (SphereLFU) to compete with the static volume maximization of FGRVB. The optimal online variant is highly workload-dependent: WildChat benefits from DALFU's aging, ELI5 from RAP's probabilistic admission, and MMLU from LRU-K’s filtering of "one-hit wonders." For the Quora dataset, which exhibits an order of magnitude lower hit rate, SurprisalLFU's use of static linguistic features to identify cache-worthy prompts more effectively than frequency or recency alone.

FGRVB and RGRVB represent distinct optimization philosophies: maximizing total future coverage versus maximizing temporal precision. FGRVB treats caching as a Maximum Coverage Problem, effectively retaining "heavy hitter" vectors that cover high future volumes regardless of temporal distance. Conversely, RGRVB optimizes for the immediate next hit; while this excels in bursty workloads and avoids cache pollution, its "myopic" view causes it to plateau below FGRVB on static distributions. 

Finally, CRVB attempts to reduce semantic caching to an exact-match problem via pre-clustering. While theoretically sound, it struggles with the "overlapping cluster" phenomenon inherent to high-dimensional spaces. As $D_{\text{thresh}}$ increases, queries often fall within the radii of multiple potential clusters, breaking the equivalence class assumption. Flexible greedy approaches (FGRVB and RGRVB) thus outperform CRVB in high-overlap regimes. In summary, FGRVB serves as the strongest approximation for the hit rate upper bound, while SphereLFU offers the best practical online approximation. At lower thresholds, SurprisalLFU benefits from breaking frequency ties using surprisal.

\subsection{Semantic Accuracy}
\label{sec:semantic_accuracy}

In semantic caching, a hit is not a binary success; the proximity of the retrieved vector to the query also indicates qualitative superiority. We quantify this using mean hit distance (MHD), the average distance between queries and their retrieved cached vectors.

Lower MHD values indicate higher semantic fidelity, as the cache consistently provides results closer to the query centroid. As shown in \Cref{fig:mean-hit-distance}, SphereLFU demonstrates a dominant advantage, achieving the lowest MHD across seven of the nine datasets. The exceptions, MMLU and WildChat, are characterized by a higher proportion of "one-hit wonders." WildChat is heavily recency-biased due to iterative prompt reformulations, while MMLU’s extreme topical diversity creates a sparse environment where density estimation is less effective.

SphereLFU’s superior performance stems from its soft frequency updates. By treating the cache as an online kernel density estimator, it identifies and retains "prototypes" that serve as the medoids of high-density semantic regions. Notably, SphereLFU outperforms the three VOPT heuristics (CRVB, FGRVB, and RGRVB) in semantic accuracy. This results from their differing objective functions: while VOPT variants are strictly formulated to maximize the binary hit count, often by placing vectors at the fringes of clusters to maximize covered volume, SphereLFU naturally gravitates toward the center of request clusters.

These findings suggest that SphereLFU not only maximizes hit rates but also preserves context integrity. This is critical for LLM-based systems, RAG pipelines, and KV caches, where the quality of the retrieved prompt or context directly impacts model performance. As the distance threshold increases, the advantage of SphereLFU becomes even more pronounced due to the higher density of hits per request.

\section{Discussion}

In this work, we have systematically studied replacement policies for semantic caching of LLMs and natural language based query-answering systems.
Our study has revealed several gaps in transferring knowledge about caching from exact match caching to semantic caching.
In particular, we identified that most of the tested workloads were frequency biased, so LFU variants performed best.
We have also addressed the question of how to handle the bookkeeping of cache management policies when there is more than one possible hit in the semantic cache. We implemented several novel policies to evaluate different approaches.

We have also shown that computing VOPT, the adaptation of OPT for semantic caching, is NP-hard.
Consequently, we introduced three heuristic alternatives to VOPT, named CRVB, FGRVB and RGRVB, and studied their performance.
Our study has found that for most workloads, the best clairvoyant algorithm significantly outperform the best online policies, meaning that there is still room for additional research.

\newpage

\paragraph*{Impact Statements:}
This paper presents work whose goal is to advance the field of machine learning such as improving the user's experience due to shorter average latencies and lower usage fees.
From the point of view of LLM providers, it should reduce resource consumption such as computations, memory, and communication w.r.t to a target sustainable throughput.
There are many potential societal consequences of our work, none of which we feel must be specifically highlighted here.

\bibliography{citations}
\bibliographystyle{icml2026}

\newpage
\appendix
\onecolumn

\section{Datasets}
\label{sec:all-datasets}

As mentioned above, we have selected diverse datasets for our evaluation.
These datasets are discussed below.
Table~\ref{tab:embedding_stats_transposed} exhibits key metrics that characterizes these datasets.
Additionally, Figure~\ref{fig:point_density_3x3} shows their point density distribution under various similarity threshold distances.

\paragraph{ELI5}
A collection of question–answer pairs drawn from the "Explain Like I’m Five" subreddit~\cite{ELI5}. The questions are short, usually no longer than a sentence, and expressed in casual, colloquial language. They often reflect trending topics or popular opinions, making the dataset a good source of conversational, non-expert phrasing.

\paragraph{WildChat}
Dialogues constructed from real ChatGPT logs, preserved in chronological order and grouped into sessions~\cite{WildChat}. Users frequently resubmit identical or near-identical queries, either within the same session or across simultaneous sessions, often with minor rephrasing. This behavior induces a recency bias, as queries are shaped by ongoing conversations and iterative reformulations.

\paragraph{Natural Questions}
An open-domain QA dataset designed around questions answerable using Wikipedia~\cite{NaturalQuestions}. These questions are general, objective, and less influenced by conversational context, offering a complementary contrast to more session-oriented datasets.

\paragraph{MS~MARCO}
A large-scale dataset of real Bing search queries paired with relevant passages, released by Microsoft~\cite{MSMARCO}. The queries mirror web-search style phrasing, ranging from keyword-style to natural-language questions, and provide coverage of diverse real-world information needs.

\paragraph{StackOverflow}
Programming-related questions obtained from the Google BigQuery public dataset of StackOverflow posts~\cite{StackOverflowBigQuery}. We restrict ourselves to question titles, which are typically concise, technical, and domain-specific, making them useful for evaluating embeddings on structured knowledge tasks.

\paragraph{Quora Question Pairs}
A dataset of user-submitted questions on Quora, annotated for semantic equivalence~\cite{QuoraQuestionPairs}. This property makes the dataset particularly suitable for studying semantic overlap and redundancy in embedding spaces.

\paragraph{MMLU}
A massive benchmark covering 57 diverse subjects, ranging from elementary mathematics to professional law and the humanities~\cite{MMLU}. It is designed to assess a model's breadth of world knowledge and problem-solving capabilities, providing a rigorous test for general knowledge retrieval outside of specific domain constraints.

\paragraph{HotpotQA}
A dataset focused on multi-hop reasoning, where questions cannot be answered by a single document alone but require combining information from multiple supporting sources~\cite{HotpotQA}. This structure challenges retrieval systems to handle complex dependencies and synthesize disconnected facts rather than performing simple keyword lookups.

\paragraph{TriviaQA} 
A large-scale reading comprehension dataset containing questions authored by trivia enthusiasts, paired with evidence from web search results~\cite{TriviaQA}. The questions are linguistically rich, complex, and often compositional, requiring the system to parse noisy, lengthy evidence documents to locate precise answers.

\begin{table}[t]
\centering
\begin{tabular}{lccccccccc}
\hline
  & MsMarco & WildChat & ELI5 & NQ & SO & Quora & MMLU & TriviaQA & HotPotQA \\
\hline
Cos Sim Avg & 0.0182 & 0.0484 & 0.0523 & 0.0206 & 0.0170 & 0.0451 & 0.0562 & 0.0570 & 0.0848 \\
Cos Sim Std & 0.0757 & 0.0861 & 0.0845 & 0.0871 & 0.0893 & 0.0843 & 0.0995 & 0.0919 & 0.1011 \\
L2 Mean & 1.4019 & 1.3781 & 1.3754 & 1.3978 & 1.4019 & 1.3806 & 1.3718 & 1.3715 & 1.3508 \\
L2 Std & 0.0555 & 0.0655 & 0.0633 & 0.0644 & 0.0658 & 0.0635 & 0.0759 & 0.0687 & 0.0766 \\
PCA Ent. & 7.8781 & 7.8251 & 7.7055 & 7.6928 & 7.6043 & 7.7551 & 7.5655 & 7.5658 & 7.4615 \\
Cluster AVG & 1.0203 & 1.3606 & 1.0434 & 1.2347 & 1.0077 & 1.2958 & 1.9112 & 1.0716 & 1.0081 \\
Cluster STD & 0.1491 & 2.5079 & 0.2383 & 0.6622 & 0.0896 & 0.8880 & 1.3860 & 0.2975 & 0.0925 \\
Hopkins & 0.7401 & 0.8044 & 0.7558 & 0.7897 & 0.7424 & 0.8054 & 0.8839 & 0.7551 & 0.7252 \\
\hline
\end{tabular}
\caption{Per dataset metrics for first 100K embeddings in each dataset.}
\label{tab:embedding_stats_transposed}
\end{table}

\begin{figure*}[t!]
  \centering
  \begin{subfigure}{0.31\linewidth}
    \includegraphics[width=\linewidth]{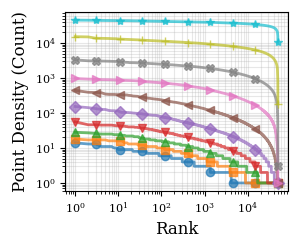}
    \caption{ELI5}
    \label{fig:eli5_density}
  \end{subfigure}\hfill
  \begin{subfigure}{0.31\linewidth}
    \includegraphics[width=\linewidth]{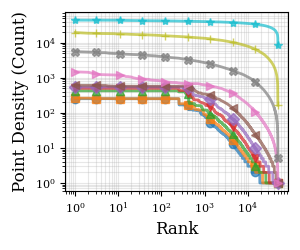}
    \caption{WildChat}
    \label{fig:wildchat_density}
  \end{subfigure}\hfill
  \begin{subfigure}{0.31\linewidth}
    \includegraphics[width=\linewidth]{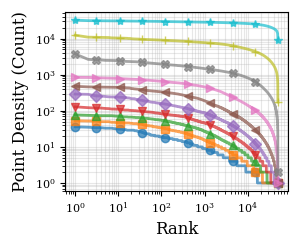}
    \caption{NaturalQuestions}
    \label{fig:nq_density}
  \end{subfigure}

  \vspace{0.5em}

  \begin{subfigure}{0.31\linewidth}
    \includegraphics[width=\linewidth]{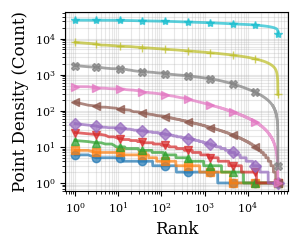}
    \caption{MS MARCO}
    \label{fig:msmarco_density}
  \end{subfigure}\hfill
  \begin{subfigure}{0.31\linewidth}
    \includegraphics[width=\linewidth]{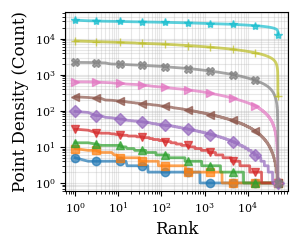}
    \caption{StackOverflow}
    \label{fig:stackoverflow_density}
  \end{subfigure}\hfill
  \begin{subfigure}{0.31\linewidth}
    \includegraphics[width=\linewidth]{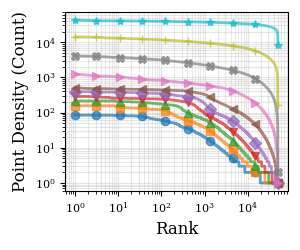}
    \caption{Quora}
    \label{fig:quora_density}
  \end{subfigure}

  \vspace{0.5em}

  \begin{subfigure}{0.31\linewidth}
    \includegraphics[width=\linewidth]{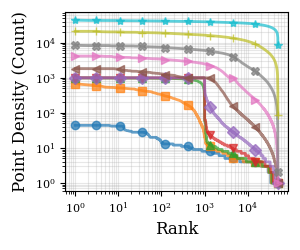}
    \caption{MMLU}
    \label{fig:mmlu_density}
  \end{subfigure}\hfill
  \begin{subfigure}{0.31\linewidth}
    \includegraphics[width=\linewidth]{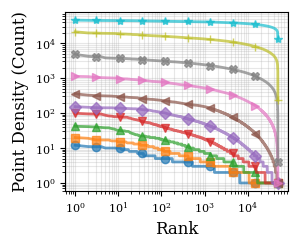}
    \caption{TriviaQA}
    \label{fig:triviaqa_density}
  \end{subfigure}\hfill
  \begin{subfigure}{0.31\linewidth}
    \includegraphics[width=\linewidth]{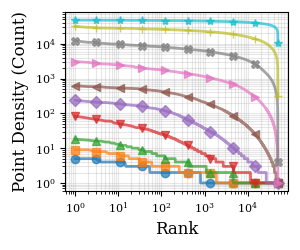}
    \caption{HotPotQA}
    \label{fig:hotpotqa_density}
  \end{subfigure}

  \caption{Point Density Distribution for Various Semantic Distance Threshold Values}
  \label{fig:point_density_3x3}
%
  \centering
    \includegraphics[width=\linewidth]{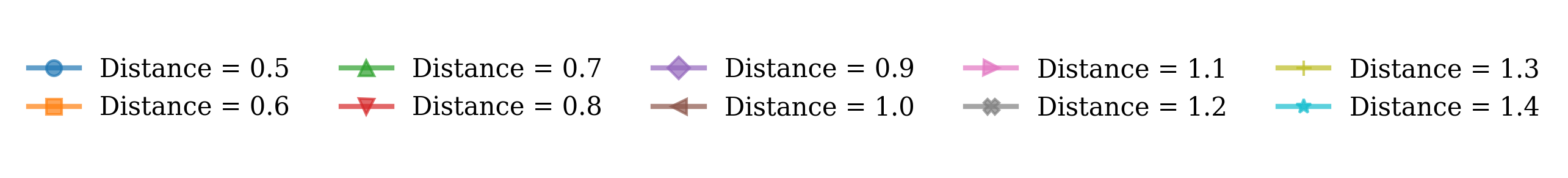}
    \label{fig:density-legend}
\end{figure*}

\subsection{Datasets Metrics}
\label{subsec:datasets-metrics}

Table~\ref{tab:embedding_stats_transposed} presents key statistical properties of the embeddings (using the first 100K queries of each dataset) to characterize the geometric structure and clustering tendency of the semantic space.
We consider the following metrics:

\begin{description}
    \item[Cos Sim (Avg/Std):] The average and standard deviation of cosine similarity between all pairs of embeddings. The low average values (near 0) indicate that the embedding space is sparsely populated, and random queries are typically orthogonal to one another.
    
    \item[L2 (Mean/Std):] The average Euclidean distance between all embeddings pairs. Since our embeddings are normalized to the unit sphere, orthogonal vectors have a distance of $\sqrt{2} \approx 1.414$. The observed means ($\approx 1.37 - 1.40$) align with the cosine similarity results, confirming that unrelated queries are far apart in the vector space.
    
    \item[PCA Ent. (Entropy):] A measure of the effective dimensionality of the data, calculated as the entropy of the explained variance ratio of the PCA components. A higher value indicates that information is distributed across many dimensions, whereas a lower value implies the data lies on a lower-dimensional manifold.
    
    \item[Cluster (AVG/STD):] We grouped vectors into semantic clusters using the threshold $D_{\text{thresh}} = 1$. \textbf{Cluster AVG} denotes the average size of a cluster; values near 1.0 (e.g., StackOverflow) indicate high uniqueness, while higher values (e.g., MMLU) indicate high redundancy. \textbf{Cluster STD} measures the variance in cluster size; a high standard deviation (e.g., WildChat) indicates a long tail distribution where specific topics are extremely popular (highly repeated) while most others are unique.
    
    \item[Hopkins Statistic:] A statistical test for spatial randomness. A value of 0.5 indicates a uniform (random) distribution, while values approaching 1.0 indicate a strong tendency to cluster. Our scores (ranging from 0.74 to 0.88) confirm that the datasets possess significant non-random semantic structure, justifying the use of semantic caching.
\end{description}

\subsection{Datasets Characteristic}
\label{subsec:datasets-Characteristic}

The empirical distributions for MMLU and WildChat, as shown in Figure~\ref{fig:point_density_3x3}, diverge significantly from the other workloads, exhibiting a distinct "step" function across varying distance thresholds. Specifically, these datasets maintain a relatively flat density for small $D_{\text{thresh}}$ values followed by a sharp slope. When cross-referenced with their exceptionally high cluster standard deviations, which are an order of magnitude larger than those of other datasets, this geometric profile suggests a high concentration of transient queries or "one-hit wonders."

While these workloads contain dense clusters of frequent queries, they are simultaneously characterized by a massive, sparse "long tail" of unique embeddings that do not overlap. For cache management, this indicates that while the "heavy hitters" are easily captured by frequency-based policies, a substantial portion of the cache is at risk of pollution by these non-reusable entries unless an effective admission filter is employed.

\section{Policies}
\label{sec:all-policies}

\subsection{Adaptions of Classic (Exact Match) Caching}

As mentioned before, numerous cache management policies have been developed over the years, e.g.,~\cite{LRU,WLFU,LFUAGING,2Q,LRUK,FRD,TinyLFU,LRFU,LIRS,DLIRS,ARC,AdaptiveTinyLFU,LeCaR,LRB,SF24,LHD}.
Here, we briefly elaborate on the ones we have experimented with in this work, which were chosen based on their popularity and represent a class of similar~policies.

\paragraph*{Least Recently Used (LRU)}
The LRU policy evicts the cache entry that has not been accessed for the longest period~\cite{LRU}. Based on the principle of time locality, it assumes that recently accessed items are more likely to be used again soon. LRU is widely used due to its effectiveness in scenarios where temporal locality is prevalent, and since it is not susceptible to Belady's anomaly~\cite{BNS69}. However, its implementation may require additional overhead to track access times or maintain an ordered list of entries.

\paragraph*{Least Frequently Used (LFU)}
The LFU policy~\cite{WLFU} removes the entry that has been accessed the least number of times. This method assumes that frequently accessed items are more likely to remain relevant. While LFU can be effective in specific scenarios, it may fail to adapt quickly to changing access patterns, especially if stale items remain in the cache due to infrequent access.

\paragraph*{LFUDA}
LFUDA~\cite{LFUAGING} enhances LFU by introducing a dynamic aging mechanism to prevent cache pollution by stale but historically popular items. It maintains a global aging factor that increases over time, ensuring that older items with high access counts are eventually evicted if not accessed recently. This approach strikes a balance between recency and frequency, adapting better to changing workloads compared to standard LFU. LFUDA is particularly useful in long running systems where workload patterns evolve over time.

\paragraph*{LRU-K}
LRU-K~\cite{LRUK} is an extension of the LRU policy that considers the \emph{K-th} most recent access when making eviction decisions. By doing so, it captures long term access patterns rather than reacting solely to the most recent usage. When \( K = 1 \), LRU-K reduces to standard LRU. As \( K \) increases, the policy becomes more resistant to short-term fluctuations, favoring items that exhibit consistent reuse over time. However, LRU-K requires additional space to maintain the history of accesses, and its performance depends on the appropriate choice of \( K \).

\paragraph*{Adaptive Replacement Cache (ARC)}
ARC~\cite{ARC} dynamically balances between recency and frequency by maintaining two LRU lists: one for recently accessed entries and one for frequently accessed entries. It uses a self tuning mechanism to adapt to workload characteristics by adjusting the size of each list based on observed hit rates. This adaptability allows ARC to outperform both LRU and LFU in many scenarios without requiring manual parameter tuning. ARC is particularly effective in systems where access patterns shift between recency and frequency dominated behaviors. LeCaR~\cite{LeCaR}, a well known ML based cache policy, demonstrated hit rates comparable to ARC in its original paper.

\paragraph*{Random Admission Policy (RAP)}
RAP~\cite{RAP} was originally proposed for top-$k$ and frequency estimation, maintaining counters for the most frequent items in a stream when the stream size far exceeds available memory. 
We adapt it for cache management. Similar to LFU, RAP identifies the least frequently used item for potential replacement. However, instead of always evicting it, the item with counter value $C$ is replaced with probability $\frac{1}{C+1}$. This probabilistic admission reduces unnecessary evictions caused by infrequent (tail) items, preserving space for items more likely to be~reused.

\subsection{Novel Cache Eviction Policies}
\label{sec:novel-policies}

As we mentioned in the Introduction, caching embedding vectors for LLM workloads presents unique challenges in applying classical cache management policies. Specifically, embedding vectors encode semantic similarity in a continuous embedding space, which allows for “close enough” hits.
However, this complicates eviction decisions. For example, two distinct queries may map to nearly identical embeddings, making it redundant to store both. Similarly, groups of related embeddings may exhibit correlated access patterns, suggesting that decisions should sometimes be made at the cluster level rather than the individual vector level. To address these challenges, we introduce several novel eviction policies that exploit semantic structure, vector similarity, and linguistic properties of the underlying data, in order to evaluate how to best exploit the properties of embedding vectors.

\subsubsection{SphereLFU}
\label{sec:SphereLFU}

SphereLFU is an online policy designed to approximate the behavior of our best offline policy, FGRVB. Recall that FGRVB maximizes the total ``volume'' of future queries covered by the cache. 
Since SphereLFU lacks future knowledge, it approximates the total volume by the integral of the probability density function (PDF) of the request stream over the covered regions.
Thus, SphereLFU reframes cache management as an online Kernel Density Estimation (KDE) task: by maintaining vectors that inhabit the highest-density regions of the embedding space, we maximize the expected number of future hits.

\paragraph{Rational}
Standard frequency-based policies (like LFU) operate on a ``winner-takes-all'' basis: only the specific vector returned as a hit receives a frequency update. In a continuous semantic space, this is suboptimal. A query falling between two cached vectors $v_1$ and $v_2$ provides evidence that \textit{both} are valuable, yet standard policies discard this information for the non-selected vector.

SphereLFU applies \textit{soft frequency updates}. We view each query not as a discrete point, but as a unit of probability mass diffusing into the local neighborhood. When a query arrives, we distribute this mass among all nearby cached vectors proportional to their proximity. This allows the cache to build a smooth density map of the query distribution, retaining groups of vectors that collectively cover high-traffic semantic regions, even if no single vector is an exact match for every~query.

\paragraph{Probabilistic Model.}
We assume each request embedding $q$ is generated by a latent semantic source $Z$, corresponding to one of the cached embeddings. Let $\{e_i\}$ denote the cached embeddings, all $\ell_2$-normalized. We model the conditional likelihood of observing $q$ given source $i$ using a spherical similarity kernel equivalent to the von Mises-Fisher distribution on the unit sphere:
\[
p(q \mid Z=i) \;\propto\; \exp\!\left(-\frac{\kappa}{2}\,d(q, e_i)^2\right),
\]
where $\kappa>0$ is a concentration parameter determining the sharpness of the kernel.

\paragraph{Soft Frequency Updates.}
Upon receiving a query $q$, we identify the set $\mathcal{M}(q)$ of cached embeddings within the threshold distance $D_{\text{thresh}}$. We calculate a \emph{responsibility} score $r_i(q)$ for each neighbor $i \in \mathcal{M}(q)$, representing the probability that vector $i$ generated query $q$:
\[
r_i(q)
=
\frac{(c_i + \alpha)\exp\!\left(-\frac{\kappa}{2} d_i^2\right)}
{\sum_{j \in \mathcal{M}(q)} (c_j + \alpha)\exp\!\left(-\frac{\kappa}{2} d_j^2\right)}
\]
where $d_i=d(q,e_i)$, $c_i$ is the current accumulated frequency of embedding $i$, and $\alpha>0$ is a smoothing constant. The counters are then updated by distributing exactly one unit of mass:
\[
c_i \leftarrow c_i + r_i(q).
\]
Unlike heuristic LFU counters, $c_i$ here represents a sufficient statistic for the mixture weight $\pi_i = \Pr[Z=i]$, converging to the true usage density under stationary conditions.

\paragraph{Distribution Shift and Decay.}
To handle non-stationary workloads where topic popularity shifts over time, we apply a multiplicative decay factor $\gamma \in (0, 1)$ during updates:
\[
c_i \leftarrow \gamma c_i + r_i(q).
\]
This effectively creates an exponentially weighted moving average of the density. For computational efficiency, rather than decaying every counter on every step, we implement this lazily or by periodically rescaling all global counters by a fixed factor (e.g., halving all counters every $N$ requests).

\paragraph{Efficiency.}

Unlike our other LFU variants that perform a $k$-NN search to update a single metadata entry, SphereLFU employs a range query to identify all vectors within the semantic threshold. 
While this theoretically allows for an update of the entire cache, the high-dimensional embedding space is characterized by extreme sparsity.
In practice, the number of neighbors within $D_{\text{thresh}}$ is consistently low, fewer than 10. 
If the neighbor count were to increase significantly, it would indicate that $D_{\text{thresh}}$ is likely set too high. 
In scenarios where strict latency bounds are required, computational complexity can be capped by restricting soft-frequency updates to the top $k$ nearest neighbors. 
This approximation preserves the most significant portion of the probability mass while maintaining a guaranteed $O(k)$ update time.

\subsubsection{MissLFU}
Similar to LFU, but only inserts an embedding vector into the cache if there are no stored semantically identical vectors in the cache.
The intuition here is that if a vector is close enough to an existing vector in the cache, there is already a (semantic) hit, so there is no point inserting the new vector.
Hence, we prefer to keep embedding vectors that may otherwise not have a cached semantically close representative.

\subsubsection{ClusterLFU}  
This policy groups vectors into clusters using a greedy online clustering algorithm (although other clustering algorithms can be used as well).
It then applies LFU at the cluster level rather than for individual vectors. 
Each cluster maintains a single LFU counter representing the aggregate access frequency of its members. 
When an eviction is required, the least frequently used cluster is selected, and a random vector from that cluster is removed.  
By aggregating counters at the cluster level, this approach reduces metadata overhead and exploits semantic similarity among vectors, allowing the cache to favor groups of related embeddings that are accessed~frequently.  

\subsubsection{ClusterLRU}  
Similar to ClusterLFU, this policy operates at the cluster level using the same greedy online clustering method. Instead of tracking frequency, it maintains the most recent access time for each cluster, based on the latest access to any of its members. Upon eviction, the least recently used cluster is selected, and a random vector from that cluster is removed.  
This aggregation reduces the need to maintain per-vector recency data, making it more efficient for large vector caches while still capturing temporal locality patterns.  

\subsubsection{DistanceLFU}  
This policy adjusts LFU counting based on vector similarity to incoming queries. If the closest stored embedding to a query is at distance \( d \) (with \( D_{\text{thresh}} \) being the maximum distance threshold to be considered a hit), the corresponding counter is incremented by:
\[
\Delta = 1 - \frac{d}{D_{\text{thresh}}}~~.
\]  
This weighting increases the counter proportionally to how close is the match, allowing the cache to prioritize embeddings that better match recent queries.  

\subsubsection{Surprisal}  
This policy evicts the sentence with the maximum surprisal. Surprisal is calculated using the word frequencies library in python, which is based on multiple data sources and includes the top 100 million most popular words. A floor value is given to missing words. The Surprisal calculation is done on the raw text string, and not on the embedding iteself.

In this paper, we've used Surprisal for efficiency reasons - as it is fast to calculate in a "bag of words" method. Yet, it is natural to utilize the perplexity of the embedding, or other metrics that can be derived from an LLM processing the query for more accurate metric.

\subsubsection{SurprisalLFU}  
This policy leverages linguistic surprisal to guide eviction decisions. Surprisal measures how unexpected a sentence is, with higher values indicating lower probability of occurrence. Perplexity is a metric closely related to surprisal, and is defined as the inverse probability of the sequence normalized by its length.  

Direct computation of perplexity from an LLM is computationally expensive for online caching. Instead, a more efficient metric can be obtained using a precomputed dictionary of token probabilities or frequency statistics from a large corpus. The surprisal of a sentence can be estimated as:  
\[
\text{Surprisal}(s) = -\sum_{w \in s} \log p(w)
\]  
where \(p(w)\) is the probability of token \(w\) under the reference distribution.  
Upon eviction, the item with the highest surprisal value out of all items with the minimal LFU counter is removed from the cache. This approach prioritizes retaining semantically common vectors that are more likely to contribute to future hits, while discarding outliers that are unlikely to be requested again. This is particularly useful, as we found that caches often accumulate a long tail of items with nearly identical counters.

\section{Vector Construction}
\label{sec:vector_construction}

To reduce the Maximum Coverage Problem (MCP) to the VOPT problem, we construct a geometric embedding that translates set membership into Euclidean distance. Let the MCP instance consist of a universe of elements $U = \{e_1, \dots, e_n\}$ and a collection of sets $\mathcal{S} = \{S_1, \dots, S_m\}$. Let $K_{\max} = \max_{j} |S_j|$ be the size of the largest set. When $K_{\max} = 1$, the problem is trivial. Without loss of generality, we assume $K_{\max} > 1$.

We define a high-dimensional Euclidean space $\mathbb{R}^{n+m}$ with orthonormal basis vectors:
\begin{itemize}
    \item $\{\mathbf{u}_1, \dots, \mathbf{u}_n\}$ corresponding to the $n$ elements.
    \item $\{\mathbf{w}_1, \dots, \mathbf{w}_m\}$ corresponding to the $m$ sets.
\end{itemize}

We construct the vectors using scaling parameters $\alpha, \beta, \gamma > 0$.
For each element $e_i \in U$, we define the \emph{element vector}:
\begin{equation}
    v_{e_i} = \alpha \cdot \mathbf{u}_i
\end{equation}
For each set $S_j \in \mathcal{S}$, we define the \emph{set vector} as a weighted combination of its members and a unique identifier:
\begin{equation}
    v_{S_j} = \sum_{e_k \in S_j} (\beta \cdot \mathbf{u}_k) + (\gamma \cdot \mathbf{w}_j)
\end{equation}

For the reduction to hold, the embedding must satisfy four conditions relative to the cache threshold $D_{\text{thresh}}$:
\begin{enumerate}
    \item \textbf{Element Separation:} Distinct elements must not cover each other.
    \item \textbf{Set Separation:} Distinct sets must not cover each other.
    \item \textbf{Membership Coverage:} A set must cover its member elements.
    \item \textbf{Non-Membership Separation:} A set must not cover non-member elements.
\end{enumerate}

These conditions correspond to the following system of inequalities regarding the squared Euclidean distance:
\begin{align}
    \|v_{e_i} - v_{e_\ell}\|^2 = 2\alpha^2 &> D_{\text{thresh}}^2 \quad \forall i \ne \ell \label{eq:sep_elem} \\
    \|v_{S_j} - v_{S_\ell}\|^2 \ge 2\gamma^2 &> D_{\text{thresh}}^2 \quad \forall j \ne \ell \label{eq:sep_set} \\
    \|v_{S_j} - v_{e_i}\|^2 \le D_{\text{thresh}}^2 &\iff e_i \in S_j \label{eq:membership}
\end{align}

Expanding the distance for a set $S_j$ and an element $e_i$, we distinguish two cases:
\begin{itemize}
    \item \textbf{Case 1 ($e_i \in S_j$):} The distance is minimized due to the shared component $\beta$:
    \begin{equation}
        (\alpha - \beta)^2 + (|S_j| - 1)\beta^2 + \gamma^2 \le D_{\text{thresh}}^2
    \end{equation}
    \item \textbf{Case 2 ($e_i \notin S_j$):} The components are orthogonal, maximizing distance:
    \begin{equation}
        \alpha^2 + |S_j|\beta^2 + \gamma^2 > D_{\text{thresh}}^2
    \end{equation}
\end{itemize}

We now prove that valid parameters $\alpha, \beta, \gamma$ always exist. We set $\beta$ such that the set vector "leans" towards its members relative to the element scale:
\begin{equation}
    \beta = \frac{\alpha}{K_{\max}}
\end{equation}
We analyze the constraints for the "worst-case" scenario (a set of maximum size $K_{\max}$). Substituting $\beta$, the squared distance for a member element becomes:
\begin{equation}
    \text{Dist}^2_{\text{in}} = \alpha^2 \left(1 - \frac{1}{K_{\max}}\right) + \gamma^2
\end{equation}
To satisfy all conditions, we must find $\alpha$ and $\gamma$ such that:
\begin{align}
    2\alpha^2 &> D_{\text{thresh}}^2 \label{eq:const1} \\
    2\gamma^2 &> D_{\text{thresh}}^2 \label{eq:const2} \\
    \alpha^2 \left(1 - \frac{1}{K_{\max}}\right) + \gamma^2 &\le D_{\text{thresh}}^2 \label{eq:const3}
\end{align}
From (\ref{eq:const3}), we derive an upper bound for $\gamma^2$:
\begin{equation}
    \gamma^2 \le D_{\text{thresh}}^2 - \alpha^2 \left(\frac{K_{\max}-1}{K_{\max}}\right)
\end{equation}
To simultaneously satisfy (\ref{eq:const2}), this upper bound must strictly exceed $D_{\text{thresh}}^2/2$:
\begin{equation}
    D_{\text{thresh}}^2 - \alpha^2 \left(\frac{K_{\max}-1}{K_{\max}}\right) > \frac{D_{\text{thresh}}^2}{2}
\end{equation}
Rearranging for $\alpha^2$, we obtain the existence condition:
\begin{equation}
    \alpha^2 < \frac{D_{\text{thresh}}^2}{2} \cdot \left( \frac{K_{\max}}{K_{\max}-1} \right)
\end{equation}
Since $K_{\max} \ge 2$ implies $\frac{K_{\max}}{K_{\max}-1} > 1$, there exists a non-empty open interval for $\alpha^2$:
\begin{equation}
    \frac{D_{\text{thresh}}^2}{2} < \alpha^2 < \frac{D_{\text{thresh}}^2}{2} \left( \frac{K_{\max}}{K_{\max}-1} \right)
\end{equation}
By choosing any $\alpha^2$ within this interval and setting $\gamma^2$ sufficiently close to (but greater than) $D_{\text{thresh}}^2/2$, all strict inequalities are satisfied. This guarantees that a valid embedding exists for any MCP instance. \qed

\begin{figure*}[ht!]
  \centering
\begin{subfigure}{0.9\linewidth}
    \includegraphics[width=\linewidth]{figures/legend.png}
    \label{fig:legend-2}
\end{subfigure}
  \vspace{0.5em}
  \begin{subfigure}{0.31\linewidth}
    \includegraphics[width=\linewidth]{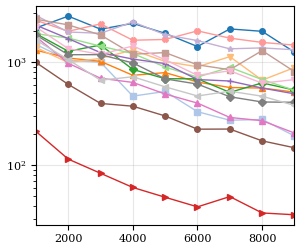}
    \caption{ELI5}
    \label{fig:Throughput:ELI5}
  \end{subfigure}\hfill
  \begin{subfigure}{0.31\linewidth}
    \includegraphics[width=\linewidth]{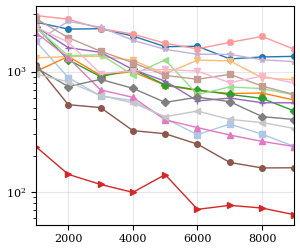}
    \caption{WildChat}
    \label{fig:Throughput:WildChat}
  \end{subfigure}\hfill
  \begin{subfigure}{0.31\linewidth}
    \includegraphics[width=\linewidth]{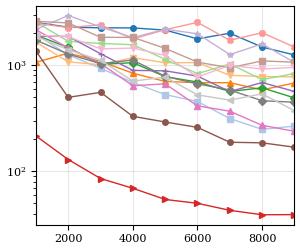}
    \caption{NaturalQuestions}
    \label{fig:fig:Throughput:NaturalQuestions}
  \end{subfigure}
  \vspace{0.5em}
    \begin{subfigure}{0.31\linewidth}
    \includegraphics[width=\linewidth]{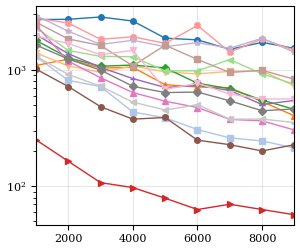}
    \caption{HotPotQA}
    \label{fig:Throughput:HotPotQA}
  \end{subfigure}\hfill
  \begin{subfigure}{0.31\linewidth}
    \includegraphics[width=\linewidth]{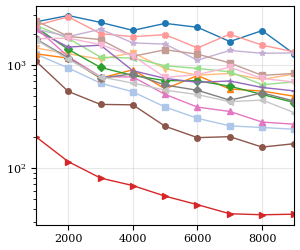}
    \caption{TriviaQA}
    \label{fig:Throughput:TriviaQA}
  \end{subfigure}\hfill
  \begin{subfigure}{0.31\linewidth}
    \includegraphics[width=\linewidth]{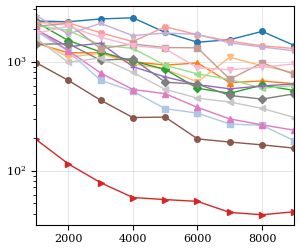}
    \caption{Quora}
    \label{fig:Throughput:Quora}
  \end{subfigure}
  \vspace{0.5em}
  \begin{subfigure}{0.31\linewidth}
    \includegraphics[width=\linewidth]{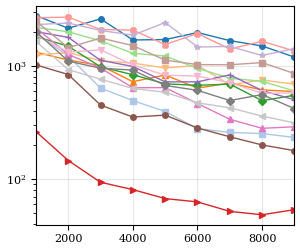}
    \caption{MsMarco}
    \label{fig:Throughput:MsMarco}
  \end{subfigure}\hfill
  \begin{subfigure}{0.31\linewidth}
    \includegraphics[width=\linewidth]{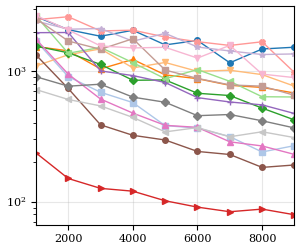}
    \caption{MMLU}
    \label{fig:Throughput:MMLU}
  \end{subfigure}\hfill
  \begin{subfigure}{0.31\linewidth}
    \includegraphics[width=\linewidth]{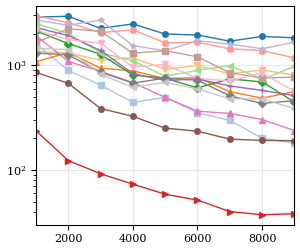}
    \caption{StackOverflow}
    \label{fig:Throughput:StackOverflow}
  \end{subfigure}
\caption{Throughput as a function of cache size - note log scale on the Y-axis}
\label{fig:Throughput}
\end{figure*}

\section{Throughput Analysis}
\label{sec:throughput}

In ~\cref{fig:Throughput} we evaluate the system throughput by measuring the average operations per second (OPS) for the combined query and update cycle across the entire request stream. Since semantic caching inherently requires identifying similar vectors to determine hits, we treat the nearest neighbor (NN) search and the subsequent cache policy logic as an atomic~operation.

Our profiling, as reported in \Cref{fig:Throughput}, indicates that the runtime is primarily dominated by the NN search latency, which is common to all semantic policies. Consequently, throughput remains consistent across different strategies, with all policies operating within the same order of magnitude. This suggests that the computational overhead introduced by the complex bookkeeping of policies like SphereLFU is negligible compared to the cost of the underlying vector retrieval.

To ensure scalability for larger cache sizes, we employed optimized implementations for all baselines. Notably, the offline clairvoyant policies (CRVB, FGRVB, RGRVB) achieve high throughput despite their algorithmic complexity. This is achieved by pre-computing a global pairwise distance matrix during initialization, effectively offloading the expensive distance calculations from the critical path of the evaluation.

FGRVB has the lowest throughput, which is caused by its implementation requiring a complex computation of the marginal gain for each vector per update.

\section{Additional Experiments}
\label{sec:additional-experiments}

We provide supplemental evaluations for varying distance thresholds, specifically $D_{\text{thresh}} = 0.5$ and $0.7$ (corresponding to cosine similarities of $0.88$ and $0.76$, respectively). These more conservative thresholds naturally result in fewer semantic hits per query. This shift diminishes the advantage of SphereLFU, which relies on a high density of neighboring vectors to effectively distribute soft-frequency updates. Conversely, SurprisalLFU exhibits relative robustness at lower thresholds. when most cached items share a baseline frequency of 1, its linguistic surprisal metric provides a more intelligent tie-breaking mechanism for eviction compared to random or recency-based selection.

\begin{figure*}[ht!]
\begin{subfigure}{0.9\linewidth}
    \includegraphics[width=\linewidth]{figures/legend.png}
    \label{fig:legend-0}
\end{subfigure}
  \vspace{0.5em}
  \centering
  \begin{subfigure}{0.31\linewidth}
    \includegraphics[width=\linewidth]{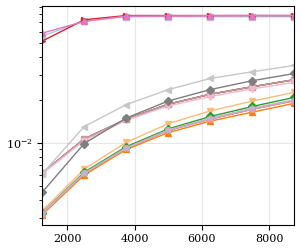}
    \caption{ELI5}
    \label{fig:ELI5}
  \end{subfigure}\hfill
  \begin{subfigure}{0.31\linewidth}
    \includegraphics[width=\linewidth]{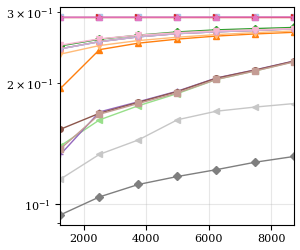}
    \caption{WildChat}
    \label{fig:WildChat}
  \end{subfigure}\hfill
  \begin{subfigure}{0.31\linewidth}
    \includegraphics[width=\linewidth]{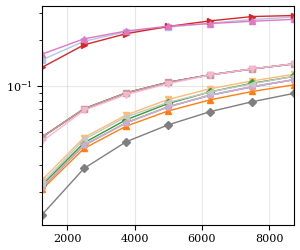}
    \caption{NaturalQuestions}
    \label{fig:NaturalQuestions}
  \end{subfigure}
  \vspace{0.5em}
    \begin{subfigure}{0.31\linewidth}
    \includegraphics[width=\linewidth]{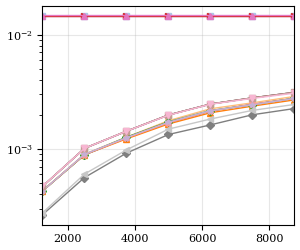}
    \caption{HotPotQA}
    \label{fig:HotPotQA}
  \end{subfigure}\hfill
  \begin{subfigure}{0.31\linewidth}
    \includegraphics[width=\linewidth]{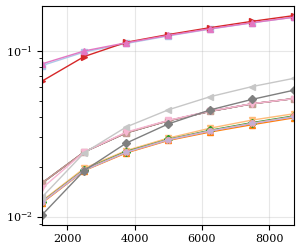}
    \caption{TriviaQA}
    \label{fig:TriviaQA}
  \end{subfigure}\hfill
  \begin{subfigure}{0.31\linewidth}
    \includegraphics[width=\linewidth]{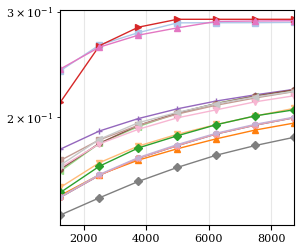}
    \caption{Quora}
    \label{fig:Quora}
  \end{subfigure}
  \vspace{0.5em}
  \begin{subfigure}{0.31\linewidth}
    \includegraphics[width=\linewidth]{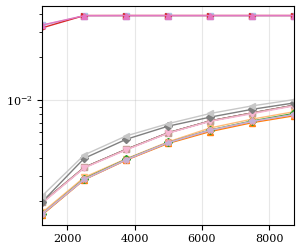}
    \caption{MsMarco}
    \label{fig:MsMarco}
  \end{subfigure}\hfill
  \begin{subfigure}{0.31\linewidth}
    \includegraphics[width=\linewidth]{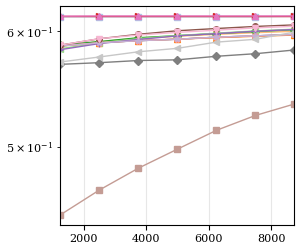}
    \caption{MMLU}
    \label{fig:MMLU}
  \end{subfigure}\hfill
  \begin{subfigure}{0.31\linewidth}
    \includegraphics[width=\linewidth]{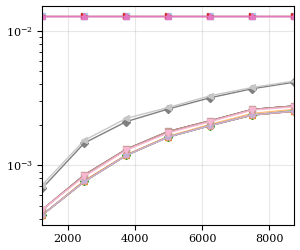}
    \caption{StackOverflow}
    \label{fig:StackOverflow}
  \end{subfigure}
\label{fig:hit-rate}
\caption{Hit rate at $D_{\text{thresh}} = 0.5$ as a function of cache size - note log scale on the Y-axis}
\end{figure*}

\begin{figure*}[ht!]
  \centering
\begin{subfigure}{0.9\linewidth}
    \includegraphics[width=\linewidth]{figures/legend.png}
    \label{fig:legend-1}
\end{subfigure}
  \vspace{0.5em}
  \begin{subfigure}{0.31\linewidth}
    \includegraphics[width=\linewidth]{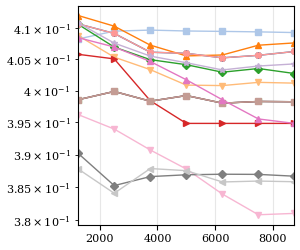}
    \caption{ELI5}
    \label{fig:MHD:ELI5}
  \end{subfigure}\hfill
  \begin{subfigure}{0.31\linewidth}
    \includegraphics[width=\linewidth]{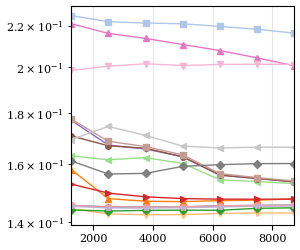}
    \caption{WildChat}
    \label{fig:MHD:WildChat}
  \end{subfigure}\hfill
  \begin{subfigure}{0.31\linewidth}
    \includegraphics[width=\linewidth]{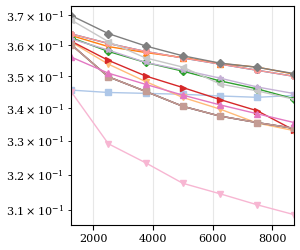}
    \caption{NaturalQuestions}
    \label{fig:fig:MHD:NaturalQuestions}
  \end{subfigure}
  \vspace{0.5em}
    \begin{subfigure}{0.31\linewidth}
    \includegraphics[width=\linewidth]{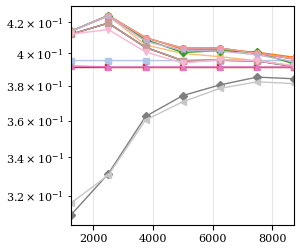}
    \caption{HotPotQA}
    \label{fig:MHD:HotPotQA}
  \end{subfigure}\hfill
  \begin{subfigure}{0.31\linewidth}
    \includegraphics[width=\linewidth]{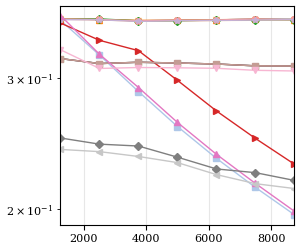}
    \caption{TriviaQA}
    \label{fig:MHD:TriviaQA}
  \end{subfigure}\hfill
  \begin{subfigure}{0.31\linewidth}
    \includegraphics[width=\linewidth]{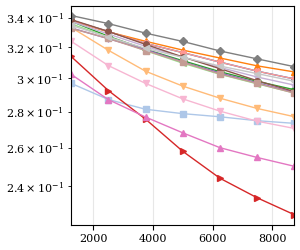}
    \caption{Quora}
    \label{fig:MHD:Quora}
  \end{subfigure}
  \vspace{0.5em}
  \begin{subfigure}{0.31\linewidth}
    \includegraphics[width=\linewidth]{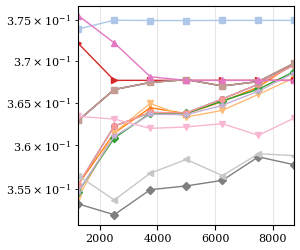}
    \caption{MsMarco}
    \label{fig:MHD:MsMarco}
  \end{subfigure}\hfill
  \begin{subfigure}{0.31\linewidth}
    \includegraphics[width=\linewidth]{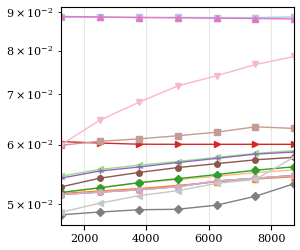}
    \caption{MMLU}
    \label{fig:MHD:MMLU}
  \end{subfigure}\hfill
  \begin{subfigure}{0.31\linewidth}
    \includegraphics[width=\linewidth]{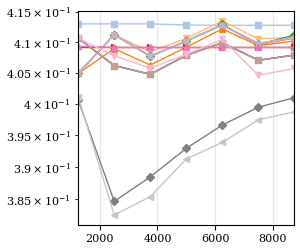}
    \caption{StackOverflow}
    \label{fig:MHD:StackOverflow}
  \end{subfigure}
\label{fig:mean-hit-distance-0.9}
\caption{Mean hit distance at $D_{\text{thresh}} = 0.5$ Distance as a function of cache size - note log scale on the Y-axis}
\end{figure*}

\begin{figure*}[ht!]
\begin{subfigure}{0.9\linewidth}
    \includegraphics[width=\linewidth]{figures/legend.png}
    \label{fig:legend-0}
\end{subfigure}
  \vspace{0.5em}
  \centering
  \begin{subfigure}{0.31\linewidth}
    \includegraphics[width=\linewidth]{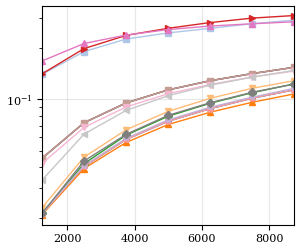}
    \caption{ELI5}
    \label{fig:ELI5}
  \end{subfigure}\hfill
  \begin{subfigure}{0.31\linewidth}
    \includegraphics[width=\linewidth]{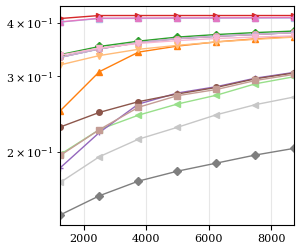}
    \caption{WildChat}
    \label{fig:WildChat}
  \end{subfigure}\hfill
  \begin{subfigure}{0.31\linewidth}
    \includegraphics[width=\linewidth]{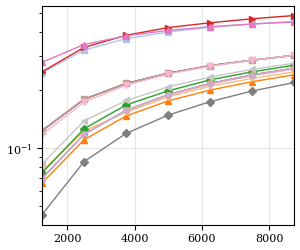}
    \caption{NaturalQuestions}
    \label{fig:NaturalQuestions}
  \end{subfigure}
  \vspace{0.5em}
    \begin{subfigure}{0.31\linewidth}
    \includegraphics[width=\linewidth]{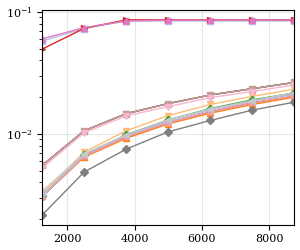}
    \caption{HotPotQA}
    \label{fig:HotPotQA}
  \end{subfigure}\hfill
  \begin{subfigure}{0.31\linewidth}
    \includegraphics[width=\linewidth]{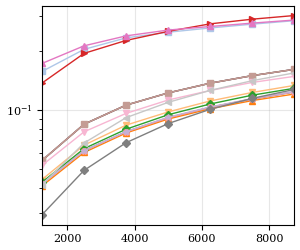}
    \caption{TriviaQA}
    \label{fig:TriviaQA}
  \end{subfigure}\hfill
  \begin{subfigure}{0.31\linewidth}
    \includegraphics[width=\linewidth]{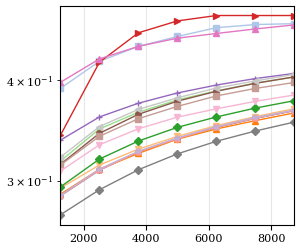}
    \caption{Quora}
    \label{fig:Quora}
  \end{subfigure}
  \vspace{0.5em}
  \begin{subfigure}{0.31\linewidth}
    \includegraphics[width=\linewidth]{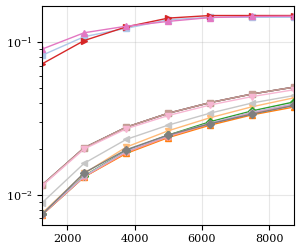}
    \caption{MsMarco}
    \label{fig:MsMarco}
  \end{subfigure}\hfill
  \begin{subfigure}{0.31\linewidth}
    \includegraphics[width=\linewidth]{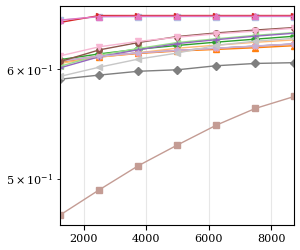}
    \caption{MMLU}
    \label{fig:MMLU}
  \end{subfigure}\hfill
  \begin{subfigure}{0.31\linewidth}
    \includegraphics[width=\linewidth]{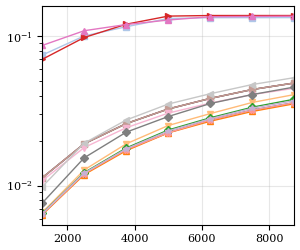}
    \caption{StackOverflow}
    \label{fig:StackOverflow}
  \end{subfigure}
\label{fig:hit-rate}
\caption{Hit rate at $D_{\text{thresh}} = 0.7$ as a function of cache size - note log scale on the Y-axis}
\end{figure*}

\begin{figure*}[ht!]
  \centering
\begin{subfigure}{0.9\linewidth}
    \includegraphics[width=\linewidth]{figures/legend.png}
    \label{fig:legend-1}
\end{subfigure}
  \vspace{0.5em}
  \begin{subfigure}{0.31\linewidth}
    \includegraphics[width=\linewidth]{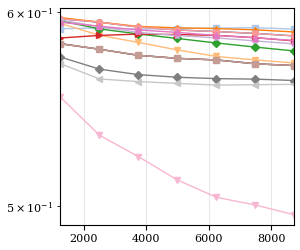}
    \caption{ELI5}
    \label{fig:MHD:ELI5}
  \end{subfigure}\hfill
  \begin{subfigure}{0.31\linewidth}
    \includegraphics[width=\linewidth]{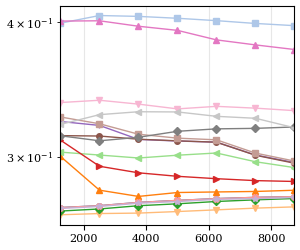}
    \caption{WildChat}
    \label{fig:MHD:WildChat}
  \end{subfigure}\hfill
  \begin{subfigure}{0.31\linewidth}
    \includegraphics[width=\linewidth]{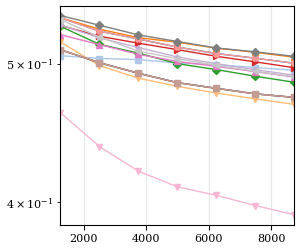}
    \caption{NaturalQuestions}
    \label{fig:fig:MHD:NaturalQuestions}
  \end{subfigure}
  \vspace{0.5em}
    \begin{subfigure}{0.31\linewidth}
    \includegraphics[width=\linewidth]{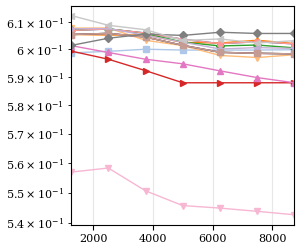}
    \caption{HotPotQA}
    \label{fig:MHD:HotPotQA}
  \end{subfigure}\hfill
  \begin{subfigure}{0.31\linewidth}
    \includegraphics[width=\linewidth]{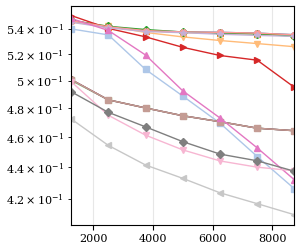}
    \caption{TriviaQA}
    \label{fig:MHD:TriviaQA}
  \end{subfigure}\hfill
  \begin{subfigure}{0.31\linewidth}
    \includegraphics[width=\linewidth]{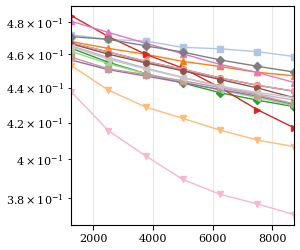}
    \caption{Quora}
    \label{fig:MHD:Quora}
  \end{subfigure}
  \vspace{0.5em}
  \begin{subfigure}{0.31\linewidth}
    \includegraphics[width=\linewidth]{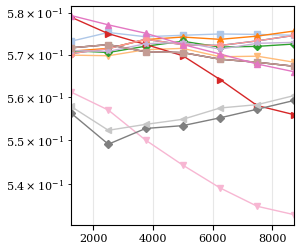}
    \caption{MsMarco}
    \label{fig:MHD:MsMarco}
  \end{subfigure}\hfill
  \begin{subfigure}{0.31\linewidth}
    \includegraphics[width=\linewidth]{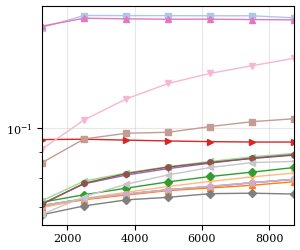}
    \caption{MMLU}
    \label{fig:MHD:MMLU}
  \end{subfigure}\hfill
  \begin{subfigure}{0.31\linewidth}
    \includegraphics[width=\linewidth]{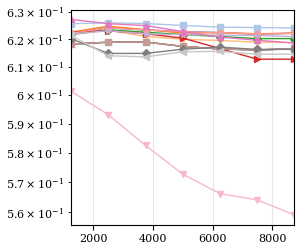}
    \caption{StackOverflow}
    \label{fig:MHD:StackOverflow}
  \end{subfigure}
\label{fig:mean-hit-distance-0.5}
\caption{Mean hit distance at $D_{\text{thresh}} = 0.7$ Distance as a function of cache size - note log scale on the Y-axis}
\end{figure*}

\end{document}